%% file: DRL_ICLR.tex
\documentclass{article} 
\usepackage{iclr2019_conference,times}

\input{math_commands.tex}

\usepackage{hyperref}
\usepackage{url}
\usepackage{amsthm}
\usepackage{csquotes}
\usepackage[boxruled]{algorithm2e}
\usepackage{xfrac}
\usepackage{enumerate}

\theoremstyle{definition}
\newtheorem{definition}{Definition}
\newtheorem{example}{Example}

\theoremstyle{plain}
\newtheorem{theorem}{Theorem}
\newtheorem{proposition}{Proposition}
\newtheorem{corollary}{Corollary}

\newcommand{\Comment}[1]{}

\DeclareMathOperator{\Supp}{supp}

\newcommand{\AP}[1]{\left(#1\right)}
\newcommand{\AB}[1]{\left[#1\right]}
\newcommand{\AC}[1]{\left\{#1\right\}}
\newcommand{\APM}[2]{\left(#1\;\middle\vert\;#2\right)}
\newcommand{\ABM}[2]{\left[#1\;\middle\vert\;#2\right]}
\newcommand{\ACM}[2]{\left\{#1\;\middle\vert\;#2\right\}}

\newcommand{\Pa}[2]{\underset{#1}{\operatorname{Pr}}\AB{#2}}
\newcommand{\CP}[3]{\underset{#1}{\operatorname{Pr}}\ABM{#2}{#3}}

\newcommand{\Ea}[2]{\underset{#1}{\operatorname{E}}\AB{#2}}
\newcommand{\CE}[3]{\underset{#1}{\operatorname{E}}\ABM{#2}{#3}}

\newcommand{\CI}[3]{\underset{#1}{\operatorname{I}}\ABM{#2}{#3}}
\newcommand{\Ia}[2]{\underset{#1}{\operatorname{I}}\AB{#2}}

\newcommand{\Ena}[1]{\operatorname{H}\AP{#1}}
\newcommand{\PS}[1]{\mathcal{P}\AP{#1}}

\newcommand{\D}{\mathrm{d}}
\newcommand{\KLD}[2]{\KL\left(#1 \mid #2\right)}

\newcommand{\Dtva}[1]{\operatorname{d}_{\textnormal{tv}}\AP{#1}}

\newcommand{\Argmax}[1]{\underset{#1}{\operatorname{arg\,max}}\,}

\newcommand{\Nats}{\mathbb{N}}

\newcommand{\Reals}{\mathbb{R}}

\newcommand{\Abs}[1]{\left\vert #1 \right\vert}

\newcommand{\Floor}[1]{\left\lfloor #1 \right\rfloor}
\newcommand{\Ceil}[1]{\left\lceil #1 \right\rceil}

\newcommand{\K}{\xrightarrow{\textnormal{k}}}

\newcommand{\A}{\mathcal{A}}
\newcommand{\St}{\mathcal{S}}
\newcommand{\T}{\mathcal{T}}
\newcommand{\Rew}{\mathcal{R}}

\newcommand{\Ut}{\operatorname{U}}
\newcommand{\V}{\operatorname{V}}
\newcommand{\Q}{\operatorname{Q}}
\newcommand{\EU}{\operatorname{EU}}
\newcommand{\Rg}{\operatorname{Reg}}

\newcommand{\MP}[2]{#1#2}
\newcommand{\MA}[2]{#1\AB{#2}}
\newcommand{\AdP}[2]{\AB{#1}#2}

\newcommand{\Tn}{\mathfrak{t}}
\newcommand{\Ad}{\upsilon}
\newcommand{\FH}{\mathcal{H}}

\newcommand{\ND}{\operatorname{D}}
\newcommand{\Rev}{\beta}
\newcommand{\RC}{\Xi}

\newcommand{\B}{\mathcal{B}}
\newcommand{\X}{\bullet}

\newcommand{\PoS}{:\St_\X^* \times \St_\X \K \A_\X}
\newcommand{\IP}{\pi^{!k}}

\title{Delegative Reinforcement Learning: learning to avoid traps with a little help}

\author{Vanessa Kosoy\\
Independent Researcher\\
Petah Tikva, Israel \\
\texttt{vanessa.kosoy@intelligence.org}
}

\iclrfinalcopy 
\begin{document}

\maketitle

\begin{abstract}
Most known regret bounds for reinforcement learning are either episodic or assume an environment without traps. We derive a regret bound without making either assumption, by allowing the algorithm to occasionally delegate an action to an external advisor. We thus arrive at a setting of active one-shot model-based reinforcement learning that we call DRL (delegative reinforcement learning.) The algorithm we construct in order to demonstrate the regret bound is a variant of Posterior Sampling Reinforcement Learning supplemented by a subroutine that decides which actions should be delegated. The algorithm is not anytime, since the parameters must be adjusted according to the target time discount. \Comment{We also demonstrate that this setting can handle situations in which the reward signal and the advisor become unreliable in particular environment states (assuming these states can be avoided.) }Currently, our analysis is limited to Markov decision processes with finite numbers of hypotheses, states and actions.
\end{abstract}

\section{Introduction}

A \emph{reinforcement learning agent} is a system that interacts with an unknown environment in a manner that is designed to maximize the expectation of a utility function that can be written as a sum of rewards over time (sometimes weighted by a time-discount function.) A standard metric for evaluating the performance of such an agent is the \emph{regret}: the difference between the expected utility of the agent in a given environment, and the expected utility of an optimal policy for the same environment. This metric allows formalizing the notion of \enquote{the agent \emph{learns} the environment} by requiring that the regret has sublinear growth in the planning horizon (usually assuming the utility function is a finite, undiscounted, sum of rewards.) For example, if we consider stateless environments, reinforcement learning reduces to a multi-armed bandit for which algorithms with guaranteed sublinear regret bounds are well-known (see e.g. \cite{Bubeck2012}.)

However, the desideratum of sublinear regret is impossible to achieve even for a finite class of environments without making further assumptions, and this is because of the possible presence of \enquote{traps}. A trap is a state which, once reached, forces a linear lower bound on regret. Consider the following example. The agent starts at state $s_1$, and as long as it takes action $a$, it receives a reward of $1$. However, if it ever takes action $b$, it will reach state $s_2$ and remain there, receiving a reward of $0$ forever, whatever it does. Thus, $s_2$ is a trap. On the other hand, it is impossible to design an algorithm which guarantees never entering traps for an arbitrary environment. For example, consider the environment that has the same structure except actions $a$ and $b$ are exchanged. In this case, if the transition matrix is not known a priori, no algorithm can learn the correct behavior, and every algorithm will have linear regret in at least one of the two environments.

There are two widespread approaches to deriving regret bounds which circumvent this problem. One is simply assuming that the environment contains no traps in some formal sense (see e.g. \cite{Nguyen2013}.) The other is \enquote{episodic learning} (see e.g. \cite{Osband2014}.) In episodic learning, the timeline is divided into intervals (\enquote{episodes}) and, either the state is assumed to reset to the initial state after each episode, or regret is defined s.t. the contribution of each episode is the difference between following the given policy and following the given policy \emph{during previous episodes} but the optimal policy in the current episode. The latter metric doesn't consider entering a trap to be a fatal event, since in the following episodes this event will be considered as \enquote{given.} That is, a policy that enters trap can still achieve sublinear regret in this sense. In fact, algorithms designed to achieve sublinear regret for sufficiently general classes of environments have the property that they eventually enter \emph{every trap they encounter} (such algorithms have a random exploration phase, like e.g. $\epsilon$-exploration in Q-learning.)

In terms of practical applications, it means that most known approaches to reinforcement learning that have theoretical performance guarantees either assume that no mistake is \enquote{fatal}, or that numerous \enquote{fatal} mistakes in the training process are acceptable. These assumptions are unacceptable in applications such as controlling a very expensive, breakable piece of machinery (e.g. spaceship) or performing a task that involves significant risk to human lives (e.g. surgery or rescue,) assuming that the algorithm \emph{cannot} be reliably trained in a simulation since the simulation doesn't reflect all the intricacies of the physical world.

This problem clearly cannot be overcome without using prior knowledge about the environment. In itself, prior knowledge is not such a strong assumption, since at least for any task that can be accomplished by a person, this prior knowledge is already available to us. The challenge is then transferring this knowledge to algorithm. This transfer can be accomplished either by manually transforming the knowledge into a formal mathematical specification, or by establishing a learning protocol that involves a human in the loop. Since human knowledge is often complex, difficult to formalise and partly intuitive, the latter option seems especially attractive.

These idea of using prior knowledge or human intervention to avoid traps has been explored by several authors (see \cite{Garcia2015} for a survey.) However, to the best of our knowledge, no previous author has established a regret bound in such a setting. In the present work, we derive such a regret bound, specifically for the setting that \cite{Clouse1997} called \enquote{ask for help} and we call \enquote{delegative reinforcement learning} (DRL), and specifically for a class of environments which consists of some finite number of Markov decision processes with a finite number of states.

In DRL, an agent interacts with an environment during an infinite sequence of \enquote{rounds}. On each round, the agent selects an action and the environment transits to a new state which is observed by the agent. The agent then receives a reward which depends on the state. There are two kinds of actions the agent can take: a \enquote{direct} action $a \in \A$ and the special delegation action $\bot$. If the agent takes action $\bot$, the \emph{advisor} takes some action $b \in \A$ which affects the environment in the same way as if it was taken directly. The agent then observes both $b$ and the new state of the environment. The utility function and regret are defined via geometric time discount with a constant $\gamma$.

The algorithm we construct in order to show the regret bound is a variant of posterior sampling reinforcement learning (see \cite{Osband2013}). Denoting $\alpha:=1-\gamma$, the timeline is divided into intervals of length $O\AP{\alpha^{-\sfrac{1}{4}}}$. At the start of each interval, the algorithm samples a hypothesis out of its current belief state, and starts carrying out an optimal policy for this hypothesis. On each round, it checks whether the desired action is known to be \enquote{safe} with high probability in a particular formal sense. If it is safe, the action is taken. If it isn't safe, delegation is performed. Moreover, the belief state evolves using all observations, but hypotheses whose probability falls below $O\AP{\alpha^{\sfrac{1}{4}}}$ are discarded altogether. We then show that (i) given relatively mild assumptions about the advisor (namely, that it only takes safe actions and it takes the optimal action with at least some small probability,) the regret is bounded by $O\AP{\alpha^{-\sfrac{3}{4}}}$\footnote{See inequality~(\ref{eqn:crl__balanced_regret_bound__regret}). In our notation, regret is normalized by a factor of $\alpha$ to lie within $[0,1]$ (see Definition~\ref{def:utility}) so the bound is $O(\alpha^{\sfrac{1}{4}})$.} (in particular it is sublinear in $\alpha^{-1}$) and (ii) the number of delegations behaves like $O\AP{\alpha^{-\sfrac{1}{4}}}$\footnote{See inequality~(\ref{eqn:crl__balanced_regret_bound__delegations}).}. Here, we only gave the dependence on $\alpha$, but the expressions we obtain are more detailed and reflect the dependence on the number of hypothesis (which we assume to be finite), the derivative of the value functions of the hypotheses and the minimal probability with which the advisor takes an optimal action.

The structure of the paper is as follows. Section~\ref{sec:results} gives all the necessary definition and formally states the results. Appendix~\ref{sec:outline} explains the algorithm implicit in the main theorem and gives an outline of the proofs. Appendix~\ref{sec:details} completes the details of the proofs. 

\section{Results}
\label{sec:results}

We start by recalling some basic definitions and properties of Markov decision processes. See e.g. \cite{Feinberg2002} for a detailed overview with proofs. First, some notation. 

Given measurable spaces $X$ and $Y$, the notation $K: X \K Y$ means that $K$ is a Markov kernel from $X$ to $Y$. Given $x \in X$, $K(x)$ is the corresponding probability measure on $Y$. Given $A \subseteq Y$ measurable, $K(A \mid x) := K(x)(A)$. Given $y \in Y$, $K(y \mid x):=K\APM{\{y\}}{x}$. Given $J: Y \K Z$, $JK: X \K Z$ is the composition of $J$ and $K$, and when $Y = X$, $K^n$ is the $n$-th composition power.

\begin{samepage}
\begin{definition}

\emph{A (finite) Markov decision process} (MDP) is a tuple

$$M:=\AP{\St_M,\ \A_M,\ s_M\in \St_M,\ \T_M: \St_M \times \A_M \K \St_M,\ \Rew_M: \St_M \rightarrow [0,1]}$$

Here, $\St_M$ is a finite set (the set of states,) $\A_M$ is a non-empty finite set (the section of actions,) $s_M$ is the initial state, $\T_M$ is the transition kernel and $\Rew_M$ is the reward function\footnote{Sometimes the reward is assumed to depend on the action as well, or on the action and the next state, but these formalisms are easily seen to be equivalent via redefinitions of the state set.}.

\end{definition}
\end{samepage}

\begin{samepage}
\begin{definition}

Given $M$ an MDP and some $\pi: \St_M \rightarrow \A_M$, we define $\T_{M\pi}: \St_M \K \St_M$ by

\begin{equation}
\T_{M\pi}(t \mid s) := \T_M\APM{t}{s,\pi(s)}
\end{equation}

That is, $\T_{M\pi}$ is the transition kernel of the Markov chain resulting from policy $\pi$ interacting with environment $M$.

\end{definition}
\end{samepage}

\begin{samepage}
\begin{definition}
\label{def:vq}

Given $M$ an MDP, we define $\V_M : \St_M \times [0,1) \rightarrow [0,1]$ and $\Q_M: \St_M \times \A_M \times [0,1) \rightarrow [0,1]$ by

\begin{equation}
\label{eqn:def__vq__v}
\V_M(s,\gamma):=(1-\gamma) \max_{\pi: \St_M \rightarrow \A_M} \sum_{n=0}^\infty \gamma^n \Ea{\T_{M\pi}^n(s)}{\Rew_M}
\end{equation}

\begin{equation}
\label{eqn:def__vq__q}
\Q_M(s,a,\gamma):=(1-\gamma)\Rew_M(s)+\gamma\Ea{t \sim \T_{M}(s,a)}{\V_M(t,\gamma)}
\end{equation}

Thus, $\V_M(s,\gamma)$ is the maximal value that can be extracted from state $s$ and $\Q_M(s,a,\gamma)$ is the maximal value that can be extracted from state $s$ after performing action $a$.

\end{definition}
\end{samepage}

\begin{samepage}
\begin{definition}

Given $M$ an MDP, we define $\V^0_M : \St_M \rightarrow [0,1]$ and $\Q^0_M: \St_M \times \A_M \rightarrow [0,1]$ by

\begin{equation}
\V_M^0(s) := \lim_{\gamma \rightarrow 1} \V_M(s,\gamma)
\end{equation}

\begin{equation}
\Q_M^0(s,a) := \lim_{\gamma \rightarrow 1} \Q_M(s,a,\gamma)
\end{equation}

The limits above are guaranteed to exist, thanks to our assumptions that $\St$ and $\A$ are finite.

\end{definition}
\end{samepage}

Given a set $A$, the notation $\PS{A}$ denotes the power set of $\A$.

\begin{samepage}
\begin{definition}

Given $M$ an MDP, we define $\A_M^0: \St_M \rightarrow \PS{\A_M}$ by

\begin{equation}
\A_M^0(s) := \Argmax{a \in \A_M} \Q_M^0(s,a)
\end{equation}

That is, $\A_M^0(s)$ is the set of actions at state $s$ that don't enter traps (i.e. destroy value in the long run.)

\end{definition}
\end{samepage}

\begin{samepage}
\begin{definition}

Given $M$ an MDP, it is well known that there are $\A_M^\star: \St_M \rightarrow \PS{\A_M}$ (the set of Blackwell optimal actions: see \cite{Feinberg2002} chapter 8) and $\gamma_M\in[0,1)$ s.t. for any $\gamma\in\AP{\gamma_M,1}$

\begin{equation}
\A_M^\star(s) = \Argmax{a \in \A_M} \Q_M\AP{s,a,\gamma}
\end{equation}

Thus, $\A_M^\star(s)$ is the set of actions that are optimal at state $s$, assuming that we plan for sufficiently long term.

\end{definition}
\end{samepage}

Given a measurable space $X$, we denote $\Delta X$ the space of probability measures on $X$. 
Given a set $A$, the notation $A^*$ will denote the set of finite strings over alphabet $A$, i.e.

\[A^* := \bigsqcup_{n = 0}^\infty A^n\]

$A^\omega$ denotes the space of infinite strings over alphabet $A$, equipped with the product topology and the corresponding Borel sigma-algebra. Given $x\in A^\omega$ and $n \in \Nats$, $x_n \in A$ is the $n$-th symbol of the string $x$ (in our conventions, $0 \in \Nats$ so the string begins from the 0th symbol.) Given $h \in A^*$ and $x \in A^\omega$, the notation $h \sqsubset x$ means that $h$ is a prefix of $x$.

Consider an MDP $M$ and some $\pi: \St_M^* \times \St_M \K \A_M$. We think of $\pi$ as a policy, where the first argument is the past history of states and the second argument is the current state. We denote $\MP{M}{\pi} \in \Delta\St_M^\omega$ the probability measure over histories resulting from policy $\pi$ interacting with environment $M$. That is, on each time step we sample an action from $\pi$ applied to previous history and last state, and sample a new state from $\T_M$ applied to last state and sampled action.

\begin{samepage}
\begin{definition}
\label{def:utility}

Given an MDP $M$ and some $\pi: \St_M^* \times \St_M \K \A_M$, we define $\Ut_M: \St_M^\omega \times [0,1) \rightarrow [0,1]$ (the utility function,) $\EU_M^\pi: [0,1) \rightarrow [0,1]$ (expected utility of policy $\pi$,) $\EU_M^\star: [0,1) \rightarrow [0,1]$ (maximal expected utility) and $\Rg_M^\pi:[0,1)\rightarrow[0,1]$ (regret of policy $\pi$) by

\begin{equation}
\Ut_M(x,\gamma) := (1-\gamma)\sum_{n=0}^\infty {\gamma^n \Rew_M\AP{x_n}}
\end{equation}

\begin{equation}
\EU_M^\pi(\gamma) := \Ea{x\sim\MP{M}{\pi}}{\Ut_M(x,\gamma)}
\end{equation}

\begin{equation}
\EU_M^\star(\gamma):=\max_{\pi: \St_M^* \times \St_M \K \A_M} {\EU_M^\pi(\gamma)} = \V_M\AP{s_M,\gamma}
\end{equation}

\begin{equation}
\Rg_M^\pi(\gamma):=\EU_M^\star(\gamma)-\EU_M^\pi(\gamma)
\end{equation}

\end{definition}
\end{samepage}

Next, we define the properties of a policy that make it a \enquote{satisfactory} advisor.

Given $X$ a topological space and $\mu$ a Borel measure on $X$, $\Supp{\mu} \subseteq X$ denotes the \emph{support} of $\mu$.

\begin{samepage}
\begin{definition}
\label{def:sane}

Consider $M$ an MDP, some $\epsilon\in(0,1)$ and some $\Ad: \St_M \K \A_M$. $\Ad$ is called \emph{$\epsilon$-sane for $M$} when for any $s \in \St_M$,

\begin{enumerate}[i.]
\item\label{con:def__sane__safe} $\Supp{\Ad(s)} \subseteq \A_M^0\AP{s}$
\item\label{con:def__sane__bold} There is $a \in \A_M^\star(s)$ s.t. $\Ad(a \mid s) > \epsilon$
\end{enumerate}

So, a policy is $\epsilon$-sane when it doesn't enter traps (destroys long-term value) and when it has a probability of more than $\epsilon$ to take a long-term optimal action.

\end{definition}
\end{samepage}

Next, we introduce a formalism describing a system of two agents where one (the \enquote{robot}) can delegate actions to another (the \enquote{advisor}.)

\begin{samepage}
\begin{definition}

Given an MDP $M$ and some $\Ad: \St_M \K \A_M$ (the advisor policy), we define the MDP $\MA{M}{\Ad}$ (the environment as perceived by the robot) by

\begin{align}
\A_{\MA{M}{\Ad}}&:=\A_M \sqcup \{\bot\} \\
\St_{\MA{M}{\Ad}}&:=\St_M \times \A_{\MA{M}{\Ad}} \\ 
s_{\MA{M}{\Ad}}&:=\AP{s_M,\bot} \\
\T_{\MA{M}{\Ad}}\APM{\AP{t,c}}{\AP{s,b},a}&:=\begin{cases} \T_M\APM{t}{s,a} \text{ if } a\ne\bot \text{ and } c=\bot \\ \T_M\APM{t}{s,c}\Ad\APM{c}{s} \text{ if } a = \bot \text{ and } c\ne\bot \\ 0 \text{ otherwise} \end{cases} \\
\Rew_{\MA{M}{\Ad}}(s,b)&:= \Rew_M(s)
\end{align}

Here, the action $\bot$ represents delegation and the $\A_{\MA{M}{\Ad}}$ factor in $\St_{\MA{M}{\Ad}}$ represents the action taken by the advisor in the last round (or $\bot$ if there was no delegation.)

\end{definition}
\end{samepage}

We will also use the following shorthand notations

\begin{samepage}
\begin{definition}
\label{def:t}

Given any MDP $M$ and $\gamma\in\AP{\gamma_M,1}$, we define

\begin{equation}
\Tn_{M}(\gamma):=\max_{s \in \St_M} \sup_{\theta\in(\gamma,1)} \Abs{\frac{\D{\V_{M}(s,\theta)}}{\D{\theta}}}
\end{equation}

The above quantity is closely related to the \emph{bias span} parameter, which is known to figure in regret bounds in the no-traps setting (see \cite{Bartlett2009}). Intuitively, it measures how costly can a non-fatal error be (the normalized value lost as a result of such an error is approximately bounded by $(1-\gamma)\Tn_{M}(\gamma)$). It can also be related to the mixing time of the Markov chain resulting from following the optimal policy in the MDP (if $P$ is the maximal period of the chain, and the total variation distance from equilibrium falls as $F\lambda^n$, then $\Tn_M\leq F\frac{1+\lambda}{1-\lambda}+P$), but discussing this in detail is out of the present scope.

\end{definition}
\end{samepage}

\begin{samepage}
\begin{definition}

Given any MDP $M$, we define $\ND_M: \AP{\St_M \times \AP{\A_M \sqcup \{\bot\}}}^\omega \rightarrow \Nats$ by

\begin{equation}
\ND_M(x) := \Abs{\ACM{n\in\Nats}{x_n\in\St_M\times\A_M}} 
\end{equation}

We think of $\ND_M(x)$ as the number of delegations in an infinite history $x$ of the MDP $\MA{M}{\Ad}$ for some $\Ad$.

\end{definition}
\end{samepage}

We can now formulate the main theorem.

For any $n \in \Nats$, we use the notation 

\[[n]:=\ACM{m\in\Nats}{m < n}\] 

We also denote 

\[\Nats^+:=\AC{n \in \Nats \mid n > 0}\]

\begin{samepage}
\begin{theorem}
\label{thm:regret_bound}

There is some constant $C \in (0,\infty)$ s.t. the following holds. Fix some $\epsilon,\eta \in (0,1)$, $T\in\Nats^+$, non-empty finite sets $\St,\A$, some $s_0 \in \St$ and some $\Rew: \St \rightarrow [0,1]$. Consider some $N \in \Nats$ which is $\geq 2$, $\AC{\T^k: \St \times \A \K \St}_{k\in[N]}$ and $\AC{\Ad^k :\St \K \A}_{k\in[N]}$. We regard the pairs $(\T^k,\Ad^k)$ as the set of \emph{hypotheses}, where $\T^k$ represents the transition kernel and $\Ad^k$ the advisor policy. Assume that for each $k\in[N]$, $\Ad^k$ is $\epsilon$-sane for the MDP $M^k:=\AP{\St,\A,s_0,\T^k,\Rew}$. Denote $\A_\X:=\A\sqcup\{\bot\}$ and $\St_\X:=\St \times \A_\X$. Fix some $\gamma\in(0,1)$ s.t. for each $k \in [N]$, $\gamma_{M^k} < \gamma$. Also, denote $L^k:=\MA{M^k}{\Ad^k}$ and $\bar{\Tn}:=\frac{1}{N}\sum_{k=0}^{N-1} \Tn_{M^k}(\gamma)$ (see Definition~\ref{def:t}). Then, there is $\pi^\dagger: \St_\X^* \times \St_\X \K \A_\X$\footnote{$\pi^\dagger$ implicitly depends on $\gamma$: in this sense, it is \emph{not} anytime. It also depends on $\eta$, $T$ and the set of hypotheses.} s.t.

\begin{equation}
\label{eqn:thm__regret_bound__regret}
\frac{1}{N}\sum_{k=0}^{N-1}\Rg_{L^k}^{\pi^\dagger}(\gamma) \leq C\AP{\eta N+\frac{\bar{\Tn}}{T}+\sqrt{\frac{(1-\gamma) T \ln{N}}{\eta}}+\frac{(1-\gamma)T \ln{N}}{\eta^2}\AP{\frac{1}{\epsilon}+\Abs{\A}}}
\end{equation}


\begin{equation}
\label{eqn:thm__regret_bound__delegations}
\forall K \in \Nats: \frac{1}{N}\sum_{k=0}^{N-1}\Pa{\MP{L^k}{\pi^\dagger}}{\ND_{M^k} > K} \leq C\AP{\eta N +\frac{\ln{N}}{K\eta}\AP{\frac{1}{\epsilon}+\Abs{\A}}}
\end{equation}

\end{theorem}
\end{samepage}

That is, we have a Bayesian regret bound for learning the true MDP starting from a prior that is a uniform distribution over $N$ hypotheses, each of which is a joint hypothesis about the transition kernel and the advisor. The bound is formulated in terms of $N$. It trivially implies a worst-case regret bound as well, at the cost of another factor of $N$. No doubt it is possible to derive other type of regret bounds for the DRL setting, e.g. in terms of the number of states and actions, but we leave it for future work.

Observe that Theorem~\ref{thm:regret_bound} is non-trivial even without equation~\ref{eqn:thm__regret_bound__delegations}, since, Definition~\ref{def:sane} is s.t. a policy that always delegates might fail to achieve any meaningful regret bound. Indeed, we can consider the special case of a multi-armed bandit, in which all actions are safe and therefore even the random policy is $\epsilon$-sane (as long as $\epsilon<\frac{1}{\Abs{\A}}$). Such a policy has normalized regret $\Omega(1)$, except for the degenerate case when all actions have the same reward.

Note that $\eta$ and $T$ are external parameters of the policy that we can choose however we like ($\eta$ is a probability threshold below which we stop considering hypotheses, and $T$ is the length of episodes for the purpose of posterior sampling; see appendix~\ref{sec:outline}.) Taking appropriate values (that depend on $\gamma$, $N$, $\epsilon$, $\Abs{\A}$ and $\bar{\Tn}$; when $\gamma$ approaches $1$, $\eta$ should fall as $(1-\gamma)^{\frac{1}{4}}$ and $T$ should grow as $(1-\gamma)^{-\frac{1}{4}}$) yields the following

\begin{samepage}
\begin{corollary}
\label{crl:balanced_regret_bound}

There is some constant $C \in (0,\infty)$ s.t. the following holds. Assume the setting of Theorem~\ref{thm:regret_bound}. Assume further that 

\begin{equation}
\label{eqn:crl__balanced_regret_bound__gamma}
\gamma \geq 1 - \frac{\AP{\bar{\Tn}+1}^3}{N^2 \ln{N}}\cdot\min\AP{\epsilon,\frac{1}{\Abs{\A}}}
\end{equation}

Denote

\begin{equation}
\label{eqn:crl__balanced_regret_bound__RC}
\RC:=\AP{N^6 \AP{\ln N} \AP{\frac{1}{\epsilon}+\Abs{\A}}\AP{\bar{\Tn}+1}}^{\sfrac{1}{4}}
\end{equation}

Then, there is $\pi^\dagger: \St_\X^* \times \St_\X \K \A_\X$ s.t. for any $k\in[N]$

\begin{equation}
\label{eqn:crl__balanced_regret_bound__regret}
\Rg_{L^k}^{\pi^\dagger}(\gamma) \leq C\RC(1-\gamma)^{\sfrac{1}{4}}
\end{equation}

\begin{equation}
\label{eqn:crl__balanced_regret_bound__delegations}
\forall K \in \Nats: \Pa{\MP{L^k}{\pi^\dagger}}{\ND_{M^k} > K} \leq C\AP{\Xi(1-\gamma)^{\sfrac{1}{4}}+\frac{1}{K}\AP{\frac{N^6\AP{\ln{N}}^3}{1-\gamma}\AP{\frac{1}{\epsilon}+\Abs{\A}}^3}^{\sfrac{1}{4}}}
\end{equation}

\end{corollary}
\end{samepage}

\Comment{The meaning of the factor $\bar{\Tn}+1$ in equation~(\ref{eqn:crl__balanced_regret_bound__RC}) might seem somewhat unclear, however, we can upper bound this factor in terms of parameters related to \emph{mixing time}: see Appendix~\ref{sec:mixing_time}.}

\appendix

\section{Proof Outline}
\label{sec:outline}

We start by giving an explicit description of an algorithm that implements the policy $\pi^\dagger$.

By condition~\ref{con:def__sane__bold} of Definition~\ref{def:sane}, for each $k\in[N]$ we can choose some $\pi^{\star k}: \St \rightarrow \A$ s.t. for any $s \in \St$, $\pi^{\star k}(s) \in \A_{M^k}^\star(s)$ and $\Ad^k\APM{\pi^{\star k}(s)}{s} > \epsilon$. The algorithm is then a variant of posterior sampling reinforcement learning in time intervals of size $T$ (see \cite{Osband2013}), where sampling hypothesis $k$ leads to using policy $\pi^{\star k}$ but delegating when we are uncertain the action is safe. Also, we repeatedly discard hypotheses with probability below $\eta$ from our belief distribution. If the currently sampled hypothesis is discarded, the algorithm continues to select safe actions until the end of the time interval, delegating whenever no action is certainly safe.\pagebreak

\RestyleAlgo{boxed}
\LinesNumbered
\DontPrintSemicolon
\SetKwFor{Loop}{InfiniteLoopBegin}{}{InfiniteLoopEnd}
\SetKwData{Z}{belief}
\SetKwData{J}{hypothesis}
\SetKwData{S}{state}
\SetKwData{RA}{agentAction}

\begin{algorithm}[H]

\SetKwData{NS}{newState}
\SetKwData{AA}{advisorAction}
\SetKwData{GA}{isSafeAction}

\S$\leftarrow s_0$\;
\Z$\leftarrow$ uniform distribution over $[N]$\;
\Loop{}{\label{ln:loop_begin}
        \J$\leftarrow$ sample the distribution \Z\;\label{ln:sample_hypothesis}
        \For{$m=0$ \KwTo $T-1$}{
                \eIf{$\Z(\J) > 0$}{
                        \RA$\leftarrow\pi^{\star\J}\AP{\S}$\;
                        \For{$k=0$ \KwTo $N-1$}{
                                \If{$\Z(k) > 0$ \bf{and} $\Ad^k\APM{\RA}{\S}=0$}{\label{ln:delegate_cond}
                                        \RA$\leftarrow\bot$\label{ln:delegate_primary}
                                }
                        }                        
                }{
                        \RA$\leftarrow\bot$\;\label{ln:delegate_secondary}
                        \For{$a\in\A$}{
                                \GA$\leftarrow$ TRUE\;
                                \For{$k=0$ \KwTo $N-1$}{
                                        \If{$\Z(k) > 0$ \bf{and} $\Ad^k\APM{a}{\S}=0$}{\label{ln:unsafe_cond}
                                                \GA$\leftarrow$ FALSE
                                        }
                                }
                                \If{\GA}{
                                        \RA$\leftarrow a$
                                }
                        }
                }
                take action \RA\;
                $\AP{\NS,\AA}\leftarrow$ make observation\;
                \For{$k=0$ \KwTo $N-1$}{
                   $\Z(k) \leftarrow \Z(k)\cdot\T_{L^k}\APM{\NS,\AA}{\S,\RA}$\;
                }
                \Z$\leftarrow\AP{\sum_{k=0}^{N-1}\Z(k)}^{-1}\cdot\Z$\;\label{ln:first_normalization}
                \For{$k=0$ \KwTo $N-1$}{\label{ln:discard_low_probability_begin}
                   \If{$\Z(k) < \eta$}{
                        $\Z(k) \leftarrow 0$\label{ln:discard}
                   }
                }
                \Z$\leftarrow\AP{\sum_{k=0}^{N-1}\Z(k)}^{-1}\cdot\Z$\;\label{ln:second_normalization}
                \S$\leftarrow$\NS
        }
}

\end{algorithm}

Note that, when the algorithm references $\Ad^k$ one lines \ref{ln:delegate_cond} and \ref{ln:unsafe_cond}, it \emph{doesn't} mean delegation. Instead, the algorithm just examines the $k$-th \emph{hypothesis} about what the advisor may do.

The form of inequalities (\ref{eqn:thm__regret_bound__regret}) and (\ref{eqn:thm__regret_bound__delegations}) is s.t. we can assume w.l.o.g. that $\eta < \frac{1}{N}$ and $\epsilon < \frac{1}{\Abs{\A}}$. In particular, the former assumption ensures that we get no division by 0 in line~\ref{ln:second_normalization} of the algorithm. Line~\ref{ln:first_normalization} might in principle involve division by 0, in which case the behavior of the algorithm can be arbitrary. For example, we may assume that in this case \Z becomes the uniform distribution again (but it doesn't matter.)

Lines~\ref{ln:discard_low_probability_begin}-\ref{ln:second_normalization} discard hypotheses that are too unlikely in order for the agent to take calculated risks (take an action even when there is a small probability of it being unsafe). Technically, in the proof they are necessary in order to apply certain mutual information inequalities (see below.) On the other hand, we also need \Z to coincide with the actual posterior given all observations, which seems like a contradiction. In order to resolve this, we introduce a class of imaginary environments $\AC{L^{k!}}_{k\in[N]}$ in which there is an additional observed signal $\Rev$ taking values in $[N]\sqcup\{\bot\}$ that, in environment $L^{k!}$, takes the value $k$ when $\DataSty{belief}(k) < \eta$ and $\bot$ otherwise. Lines \ref{ln:discard_low_probability_begin}-\ref{ln:second_normalization} then correspond to conditioning \Z on the observation $\beta=\bot$. That is, in the imaginary setting these lines are replaced by the following:

\begin{algorithm}[h]

\setcounter{AlgoLine}{32}

$\beta\leftarrow$ observe

\eIf{$\beta = \bot$}{
        \For{$k=0$ \KwTo $N-1$}{
                \If{$\Z(k) < \eta$}{
                        $\Z(k) \leftarrow 0$\label{ln:beta__discard}
                }
        }
        \Z$\leftarrow\AP{\sum_{k=0}^{N-1}\Z(k)}^{-1}\cdot\Z$\;
}{
        \Z$ \leftarrow \boldsymbol{0}$\;
        $\Z(\beta) \leftarrow 1$\;
}

\end{algorithm}

We will thereby derive the regret bound by (i) deriving a regret bound in the imaginary setting and (ii) bounding the difference between the imaginary setting and the real setting.

Given an MDP $M$ and any $\pi: \St_M^* \times \St_M \K \A$, we define $\FH_{M\pi} \subseteq \St_M^*$ by

\begin{equation}
\FH_{M\pi} := \ACM{h \in \St_M^*}{\Pa{x\sim M\pi}{h \sqsubset x} > 0}
\end{equation}

Observe that, in the imaginary setting, the policy $\IP: \St_\X^* \times \St_\X \K \A_\X$ implemented by our algorithm (which depends explicitly on $k$ because $k$ determines $\beta$) has the property

\begin{equation}
\forall h\in\FH_{L^k\pi^{!k}},s\in\St_\X: \Supp{\IP(h,s)} \subseteq \A_{M^k}^0(s)\cup\{\bot\}
\end{equation}

This is thanks to the condition at line~\ref{ln:delegate_cond} and property~\ref{con:def__sane__safe} of Definition~\ref{def:sane}.

Combining $\pi^{!k}$ with the advisor $\Ad^k$ we get the policy $\AdP{\Ad^k}{\IP}: \St^*\times\St\K\A$ which satisfies (using property~\ref{con:def__sane__safe} of Definition~\ref{def:sane} again)

\begin{equation}
\label{eqn:imaginary_safety}
\forall h\in\FH_{M^k,\AdP{\Ad^k}{\IP}},s\in\St: \Supp{\AdP{\Ad^k}{\IP}(h,s)} \subseteq \A_{M^k}^0(s)
\end{equation}

The regret incurred during each \enquote{episode} of length $T$ can be divided into short-term (associated with the rewards during the episode) and long-term (associated with the rewards after the episode, or, equivalently, with the value of the state reached at the end of the episode.) To describe the short-term regret, we introduce the policies $\AC{\pi_n^{\star k}: \St^* \times \St \K \A}_{n \in \Nats}$ defined by

\begin{equation}
\pi_n^{\star k}(h,s):=\begin{cases} \AdP{\Ad^k}{\IP}(h,s) \text{ if } \Abs{h} < nT \\ \pi^{\star k}(s) \text{ otherwise} \end{cases}
\end{equation}

Here, $\Abs{h}$ denotes the length of $h$. That is, for $h \in \St^m$, $\Abs{h}:=m$.

Define $\Rew_\X: \St_\X \rightarrow [0,1]$ by $\Rew_\X(s,a):=\Rew(s)$. For each $k\in[N]$ and $n\in\Nats$, define $\EU_n^{\star k},\EU_n^{!k}\in[0,1]$ by

\begin{equation}
\label{eqn:eustar}
\EU_n^{\star k}:=\frac{1-\gamma}{1-\gamma^T}\sum_{m=0}^{T-1} \gamma^m \Ea{x\sim{\MP{M^k}{\pi_n^{\star  k}}}}{\Rew\AP{x_{nT+m}}}
\end{equation}

\begin{equation}
\label{eqn:eushriek}
\EU_n^{!k}:=\frac{1-\gamma}{1-\gamma^T}\sum_{m=0}^{T-1} \gamma^m \Ea{x\sim{\MP{L^k}{\IP}}}{\Rew_\X\AP{x_{nT+m}}}
\end{equation}

Due to equation~(\ref{eqn:imaginary_safety}), the long-term regret per episode is $O\AP{\Tn_{M^k}(\gamma)\cdot(1-\gamma)}$. The number of episodes that are significant in terms of time discount is $\frac{1}{(1-\gamma)T}$. Therefore, the total contribution of the long-term regret is $O\AP{\frac{\Tn_{M^k}(\gamma)}{T}}$. This gives us\footnote{See Proposition~\ref{prp:short_long} for the detailed derivation.}

\begin{equation}
\label{eqn:short_long_decomposition}
\Rg_{L^k}^{\IP}(\gamma) = \AP{1-\gamma^T}\sum_{n=0}^\infty \gamma^{nT}\AP{\EU^{\star k}_n-\EU^{!k}_n}+O\AP{\frac{\Tn_{M^k}(\gamma)}{T}}
\end{equation}

In order to further analyze the short-term regret, we introduce the policies\\$\AC{\pi_n^{\sharp k}: \St^* \times \St \K \A}_{k \in [N],n \in \Nats}$. These policies result from modifying the algorithm as follows (starting from line~\ref{ln:loop_begin}.)

\begin{algorithm}[h]

\setcounter{AlgoLine}{2}

\SetKwData{Counter}{episodeNumber}

\Counter$\leftarrow$ 0\;

\Loop{}{
        \eIf{$\Counter < n$}{
                \J$\leftarrow$ sample the distribution \Z
        }{
                \J$\leftarrow k$
        }
        \ldots\;
        \setcounter{AlgoLine}{51}
        \Counter$\leftarrow \Counter + 1$
}

\end{algorithm}

We also define $\EU^{\sharp k}_n\in[0,1]$ by

\begin{equation}
\label{eqn:eusharp}
\EU_n^{\sharp k}:=\frac{1-\gamma}{1-\gamma^T}\sum_{m=0}^{T-1} \gamma^m \Ea{x\sim{\MP{L^k}{\pi_n^{\sharp  k}}}}{\Rew_\X\AP{x_{nT+m}}}
\end{equation}

We can now rewrite $\EU^{\star k}_n-\EU^{!k}_n$ as $\AP{\EU^{\star k}_n-\EU^{\sharp k}_n}+\AP{\EU^{\sharp k}_n-\EU^{!k}_n}$ and bound the contribution of each term separately.

The difference between $\pi^{\star k}_n$ and $\pi^{\sharp k}_n$ is that the latter sometimes delegates even in the $n$-th (and later) episodes\footnote{Technically, $\pi^{\star k}_n$ is a policy for $M^k$ so it's not strictly meaningful to say it \enquote{delegates} at all, but we think of it as delegating when the $\pi^{!k}$ \enquote{subroutine} inside it is called and delegates.}. Therefore, we can bound the difference in expected utilities by bounding the expected number of delegations. Now, delegation is only performed in one of two scenarios, corresponding to line~\ref{ln:delegate_primary} and line~\ref{ln:delegate_secondary}. In the scenario of line~\ref{ln:delegate_primary}, we have the action $\pi^{\star\J}(\S)$ which, with probability at least $\eta$ over hypotheses is taken with probability at least $\epsilon$ by the advisor (at least $\eta$ since this is the minimal value $\Z(\J)$ can have.) On the other hand, with probability at least $\eta$ over hypotheses, the same action is never taken by the advisor (otherwise we wouldn't delegate.) Therefore, observing whether this action is taken by the advisor provides an amount of information about the environment that can be bounded below in terms of $\eta$ and $\epsilon$. In the scenario of line~\ref{ln:delegate_secondary}, there is no action which is known with probability at least $1-\eta$ over hypotheses to be taken by the advisor with positive probability. Since observing the action actually taken by the advisor provides an example of an action which had positive probability, we gain an amount of information that can be bounded from below in terms of $\eta$. In both cases, we can show information gain is $\Omega(\eta\epsilon)$ (see Proposition~\ref{prp:delegation_information}\footnote{We think of $K$ as the (unknown) correct hypothesis, $X$ as the advisor action and $a^*$ as $\pi^{\star\J}(\S)$.}.) Since the initial entropy is $\ln{N}$, this means that the number of delegations is $O\AP{\frac{\ln{N}}{\eta\epsilon}}$ (see Proposition~\ref{prp:delegation}\footnote{We think of $\bar{\Theta}_n$ as the information used to compute \Z (including \S,) $\Psi_n$ as $\pi^{\star\J}(\S)$ when $\Z(\J)>0$ and $\bot$ otherwise, and $Z_n$ as \Z.}.) This is similar to the analysis done in \cite{Russo2016} for ordinary Thompson sampling.

Now we use this to bound $\Abs{\EU^\star_n-\EU^\sharp_n}$. We have

\begin{equation}
\Abs{\EU^{\star k}_n-\EU^{\sharp k}_n} \leq \Pa{x\sim L^k\pi^{\sharp k}_n}{\exists m \in [T]: x_{nT+m+1}\in\St \times \A}
\end{equation}

We bounded the expected number of delegations for $\pi^{!k}$ but not $\pi^{\sharp k}_n$. Since $\pi^{\sharp k}_n$ differs from $\pi^{!k}$ only by always selecting the correct hypothesis at the $n$-th and further episodes, and since the probability of selecting the correct hypothesis at line~\ref{ln:sample_hypothesis} is at least $\eta$, we get

\begin{align*}
\Abs{\EU^{\star k}_n-\EU^{\sharp k}_n} &\leq \frac{1}{\eta}\Pa{x\sim L^k\pi^{!k}}{\exists m \in [T]: x_{nT+m+1}\in\St \times \A} \\ 
&\leq \frac{1}{\eta}\Ea{x\sim L^k\pi^{!k}}{\Abs{\AC{m \in [T]: x_{nT+m+1}\in\St \times \A}}}
\end{align*}

\begin{equation}
\label{eqn:sharp_vs_star_via_delegations}
\sum_{n=0}^\infty {\Abs{\EU^{\star k}_n-\EU^{\sharp k}_n}} \leq \frac{1}{\eta}\Ea{x\sim L^k\pi^{!k}}{\Abs{\AC{n \in \Nats: x_n\in\St \times \A}}}
\end{equation}

\begin{equation}
\label{eqn:sharp_vs_star}
\frac{1}{N}\sum_{k=0}^{N-1}\sum_{n=0}^\infty {\Abs{\EU^{\star k}_n-\EU^{\sharp k}_n}} = O\AP{\frac{\ln{N}}{\eta^2\epsilon}}
\end{equation}

Observing the rewards received during an episode yields information about the environment. The expected information gain can only vanish when you expect to receive the same rewards regardless of which hypothesis is correct, in which case the policy $\pi^{\star\J}$ is optimal regardless of $\J$. This allows us to derive a lower bound for the information gain in terms of the difference between the rewards received by $\pi^{!k}$ and $\pi^{\sharp k}_n$. Denoting $I_n$ the expected information gain in episode $n$, we have (see Proposition~\ref{prp:thompson} and further details in Appendix~\ref{sec:details})

\begin{equation}
\frac{1-\gamma^T}{N}\sum_{k=0}^{N-1}\sum_{n=0}^\infty \gamma^{nT}\AP{\EU^{\sharp k}_n-\EU^{!k}_n} \leq \sqrt{\frac{1-\gamma^T}{2\eta}\sum_{n=0}^\infty \gamma^{nT}I_n}
\end{equation}

Using once again the fact that the initial entropy is $\ln{N}$, this implies

\begin{equation}
\label{eqn:shriek_vs_sharp}
\frac{1-\gamma^T}{N}\sum_{k=0}^{N-1}\sum_{n=0}^\infty \gamma^{nT}\AP{\EU^{\sharp k}_n-\EU^{!k}_n} \leq \sqrt{\frac{\AP{1-\gamma^T}\ln{N}}{2\eta}}
\end{equation}

Combining inequalities (\ref{eqn:short_long_decomposition}), (\ref{eqn:sharp_vs_star}) and (\ref{eqn:shriek_vs_sharp}), we get

\begin{equation}
\label{eqn:regret_shriek}
\frac{1}{N}\sum_{k=0}^{N-1}\Rg_{L^k}^{\IP}(\gamma) = O\AP{\frac{\bar{\Tn}}{T}+\sqrt{\frac{\AP{1-\gamma^T}\ln{N}}{\eta}}+\frac{\AP{1-\gamma^T}\ln{N}}{\eta^2\epsilon}}
\end{equation}

Finally, we observe that, in the real setting (without the $\beta$ signal,) line~\ref{ln:discard} can be reached at most $N-1$ times before the first division by zero at line~\ref{ln:first_normalization}. Moreover, such division by zero will never happen unless the correct hypothesis is discarded. Each time line~\ref{ln:discard} is reached, the probability that $\Z$ assigns probability below $\eta$ to the correct hypothesis is at most $\eta$. Therefore, the probability that the correct hypothesis is discarded is at most $\eta(N-1)$. This allows us to bound the total variation distance between the real setting and imaginary setting by $O(\eta N)$, producing inequality~(\ref{eqn:thm__regret_bound__regret})\footnote{The reason inequality~(\ref{eqn:thm__regret_bound__regret}) has $\frac{1}{\epsilon}+\Abs{\A}$ instead of $\epsilon$ is because we needed to assume w.l.o.g. that $\epsilon < \frac{1}{\Abs{A}}$. On the other hand, the assumption $\frac{1}{\eta}<N$ is justified by the appearance of the $\eta N$ term in the bound. Also, we can use $(1-\gamma)T$ instead of $1-\gamma^T$ because the form of the bound allows assuming w.l.o.g. that $(1-\gamma)T \ll 1$.}.

For reasons we already outlined, we have

\begin{equation}
\frac{1}{N}\sum_{k=0}^{N-1} \Ea{L^k\pi^{!k}}{\ND_{M^k}} = O\AP{\frac{\ln{N}}{\eta\epsilon}}
\end{equation}

Using Markov's inequality, we get

\begin{equation}
\forall K\in\Nats: \frac{1}{N}\sum_{k=0}^{N-1} \Pa{L^k\pi^{!k}}{\ND_{M^k}> K} = O\AP{\frac{\ln{N}}{K\eta\epsilon}}
\end{equation}

Using again the relationship we established between the real and imaginary settings, we get inequality~(\ref{eqn:thm__regret_bound__delegations})\footnote{The reason that the second term of the right hand side of inequality~(\ref{eqn:thm__regret_bound__delegations}) has $\frac{1}{\eta}+N$ instead of $\eta$ and $\frac{1}{\epsilon}+\Abs{\A}$ instead of $\epsilon$ is because we needed to assume w.l.o.g. that $\eta < \frac{1}{N}$ and $\epsilon < \frac{1}{\Abs{A}}$.}.

\section{Proof Details}
\label{sec:details}

\begin{samepage}
\begin{definition}

Given an MDP $M$ and $\pi: \St_M^* \times \St_M \K \A_M$, we define $\Q_{M\pi}: \St_M^* \times \St_M \times [0,1) \rightarrow [0,1]$ by

\begin{equation}
\Q_{M\pi}(h,s,\gamma):=\Ea{a\sim\pi(h,s)}{\Q_M(s,a,\gamma)}
\end{equation}

\end{definition}
\end{samepage}

Given a set $A$, $x\in A^\omega$ and $n\in\Nats$, the notation $x_{:n}$ will indicate the prefix of $x$ of length $n$. That is, $x_{:n} \in A^n$ and $x_{:n} \sqsubset x$.

\begin{samepage}
\begin{proposition}
\label{prp:regret_vq}

Consider an MDP $M$, $\gamma\in(0,1)$ and $\pi: \St_M^* \times \St_M \K \A_{M}$. Then,

\begin{equation}
\Rg_{M}^{\pi}(\gamma)=\sum_{n=0}^\infty {\gamma^n \Ea{x\sim M\pi}{\V_{M}\AP{x_n,\gamma}-\Q_{M\pi}\AP{x_{:n+1},\gamma}}}
\end{equation}

\end{proposition}
\end{samepage}

\begin{proof}

For the sake of encumbering the notation less, we will omit the argument $\gamma$ in functions that depend on it. We will also omit the subscript $M$ and denote $s_0:=s_M$.

For any $x \in \St^\omega$ s.t. $s_0 \sqsubset x$, it is easy to see that

$$\EU^{\star}=\V\AP{s_0}=\sum_{n=0}^\infty \gamma^n \AP{\V\AP{x_{n}}-\gamma\V\AP{x_{n+1}}}$$

$$\Ut(x)=(1-\gamma)\sum_{n=0}^\infty \gamma^n \Rew\AP{x_{n}}$$

\begin{align*}
\EU^{\star} - \Ut(x)=\sum_{n=0}^\infty \gamma^n &\AP{\V\AP{x_{n}}-(1-\gamma)\Rew\AP{x_{n}}-\gamma\V\AP{x_{n+1}}} \\ =\sum_{n=0}^\infty \gamma^n &\Big(\V\AP{x_{n}}-\Q_\pi\AP{x_{:n+1}}
\\ &+\Q_\pi\AP{x_{:n+1}}-(1-\gamma)\Rew\AP{x_{n}}-\gamma\V\AP{x_{n+1}}\Big)
\end{align*}

Taking expected value over $x$ w.r.t. $M\pi$, we get

\begin{align*}
\Rg^{\pi}=\sum_{n=0}^\infty \gamma^n \Bigg(&\Ea{M\pi}{\V\AP{x_{n}}-\Q_\pi\AP{x_{:n+1}}} \\ 
+&\Ea{M\pi}{\Q_\pi\AP{x_{:n+1}}-(1-\gamma)\Rew\AP{x_{n}}-\gamma\V\AP{x_{n+1}}}\Bigg)
\end{align*}

Equation~(\ref{eqn:def__vq__q}) implies that the second term vanishes, yielding the desired result.
\end{proof}

\begin{samepage}
\begin{proposition}
\label{prp:short_long}

Consider an MDP $M$, $\gamma\in\AP{\gamma_M,1}$, $T\in\Nats^+$, $\pi^\star: \St_M \K \A_M$ and $\pi^0: \St_M^* \times \St_M \K \A_M$. For any $n \in \Nats$, define $\pi^\star_n: \St_M^* \times \St_M \K \A_M$ by

$$\pi^\star_n(h,s):=\begin{cases} \pi^0(h,s) \text{ if } \Abs{h} < nT \\ \pi^\star\AP{s} \text{ otherwise} \end{cases}$$

Assume that 

\begin{enumerate}[i.]

\item\label{con:prp__short_long__star} For any $s \in \St_M$, $\Supp{\pi^\star(s)} \subseteq \A_M^\star(s)$.
\item\label{con:prp__short_long__zero} For any $h \in \FH_{M\pi^0}$ and $s \in \St_M$, $\Supp{\pi^0(h,s)} \subseteq \A_{M}^0\AP{s}$.

\end{enumerate}

For any $n \in \Nats$, define $\EU_n^\star,\EU_n^0\in[0,1]$ by

\begin{equation}
\label{eqn:prp__short_long__eu_star}
\EU_n^{\star}:=\frac{1-\gamma}{1-\gamma^T}\sum_{m=0}^{T-1} \gamma^m \Ea{x\sim{\MP{M}{\pi_n^{\star}}}}{\Rew\AP{x_{nT+m}}}
\end{equation}

\begin{equation}
\label{eqn:prp__short_long__eu_zero}
\EU_n^{0}:=\frac{1-\gamma}{1-\gamma^T}\sum_{m=0}^{T-1} \gamma^m \Ea{x\sim{\MP{M}{\pi^{0}}}}{\Rew\AP{x_{nT+m}}}
\end{equation}

Then,

\begin{equation}
\Rg^{\pi^0}_M(\gamma) \leq \AP{1-\gamma^T}\sum_{n=0}^\infty \gamma^{nT}\AP{\EU^\star_n-\EU^0_n} + 2\Tn_M(\gamma)\cdot\frac{\gamma^T(1-\gamma)}{1-\gamma^T}
\end{equation}

\end{proposition}
\end{samepage}

\begin{proof}

For the sake of encumbering the notation less, we will use the shorthands $\Rew_n:=\Rew_M\AP{x_n}$, $\V_n:=\V_M\AP{x_n,\gamma}$, $\V^0_n:=\V^0_M\AP{x_n}$, $\Q_{\pi n}:=\Q_{M\pi}\AP{x_{:n+1},\gamma}$ and $\Tn:=\Tn_M(\gamma)$.

By Proposition~\ref{prp:regret_vq}, for any $l \in \Nats$

$$\Rg_M^{\pi_l^\star} = \sum_{n=0}^\infty{\gamma^n \Ea{M\pi_l^\star}{\V_n-\Q_{\pi^\star_l n}}}$$

$\pi^\star_l$ coincides with $\pi^\star$ after $lT$, therefore the corresponding terms on the right hand side vanish.

$$\Rg_M^{\pi_l^\star} = \sum_{n=0}^{lT-1}{\gamma^n \Ea{M\pi^0}{\V_n-\Q_{\pi^0 n}}}$$

Subtracting the equalities for $l+1$ and $l$, we get

$$\EU_M^{\pi_{l}^\star} - \EU_M^{\pi_{l+1}^\star} = \sum_{n=lT}^{(l+1)T-1}{\gamma^n \Ea{M\pi^0}{\V_n-\Q_{\pi^0n}}}$$

$$(1-\gamma)\sum_{n=0}^\infty {\gamma^n\AP{\Ea{M\pi^\star_l}{\Rew_n}-\Ea{M\pi^\star_{l+1}}{\Rew_n}}} = \sum_{n=lT}^{(l+1)T-1}{\gamma^n \Ea{M\pi^0}{\V_n-\Q_{\pi^0n}}}$$

$\pi^\star_l$ and $\pi^\star_{l+1}$ coincide until $lT$, therefore

$$(1-\gamma)\sum_{n=lT}^\infty {\gamma^n\AP{\Ea{M\pi^\star_l}{\Rew_n}-\Ea{M\pi^\star_{l+1}}{\Rew_n}}} = \sum_{n=lT}^{(l+1)T-1}{\gamma^n \Ea{M\pi^0}{\V_n-\Q_{\pi^0n}}}$$

\Comment{Denote $\rho^*_l:=\mu\bowtie\pi^*_l$, $\rho^0:=\mu\bowtie\pi^0$. We also use the shorthand notations $r_n:=r\AP{x_{:n}}$, $\V_n(\gamma):=\V\AP{x_{:n},\gamma}$, $\Q_n(\gamma):=\Q\AP{x_{:n},x_n^\A,\gamma}$.} Both $\pi^\star_l$ and $\pi^\star_{l+1}$ coincide with $\pi^\star$ after $(l+1)T$, therefore

\begin{align*}
&(1-\gamma)\sum_{n=lT}^{(l+1)T-1} {\gamma^n\AP{\Ea{M\pi^\star_l}{\Rew_n}-\Ea{M\pi^0}{\Rew_n}}}\\ 
&+ \gamma^{(l+1)T}\AP{\Ea{M\pi^\star_l}{\V_{(l+1)T}}-\Ea{M\pi^0}{\V_{(l+1)T}}} = \sum_{n=lT}^{(l+1)T-1}{\gamma^n \Ea{M\pi^0}{\V_n-\Q_{\pi^0n}}}
\end{align*}

By the mean value theorem, for each $s\in\St_M$ we have

$$\V_M^0(s) - \Tn \cdot (1-\gamma) \leq \V_M(s,\gamma) \leq \V_M^0(s) + \Tn \cdot (1-\gamma)$$

It follows that

\begin{align*}
&(1-\gamma)\sum_{n=lT}^{(l+1)T-1} {\gamma^n\AP{\Ea{M\pi^\star_l}{\Rew_n}-\Ea{M\pi^0}{\Rew_n}}}\\ 
&+ \gamma^{(l+1)T}\AP{\Ea{M\pi^\star_l}{\V^0_{(l+1)T}}-\Ea{M\pi^0}{\V^0_{(l+1)T}}+2\Tn\cdot(1-\gamma)} \geq \sum_{n=lT}^{(l+1)T-1}{\gamma^n \Ea{M\pi^0}{\V_n-\Q_{\pi^0n}}}
\end{align*}

\Comment{Here, an expected value w.r.t. the difference of two probability measures is understood to mean the corresponding difference of expected values.}

It is easy to see that assumptions \ref{con:prp__short_long__star} and \ref{con:prp__short_long__zero} imply that $\V_n^0$ is a martingale for $M\pi^\star$ and $M\pi^0$ and therefore

$$\Ea{M\pi^\star_l}{\V^0_{(l+1)T}}=\Ea{M\pi^0}{\V^0_{(l+1)T}}=\V^0\AP{s_M}$$

We get

\begin{equation*}
(1-\gamma)\sum_{n=lT}^{(l+1)T-1} {\gamma^n\AP{\Ea{M\pi^\star_l}{\Rew_n}-\Ea{M\pi^0}{\Rew_n}}}+ 2\Tn\gamma^{(l+1)T}(1-\gamma) \geq \sum_{n=lT}^{(l+1)T-1}{\gamma^n \Ea{M\pi^0}{\V_n-\Q_{\pi^0n}}} 
\end{equation*}

Summing over $l$, we get

\[(1-\gamma)\sum_{l=0}^\infty\sum_{n=lT}^{(l+1)T-1} {\gamma^n\AP{\Ea{M\pi^\star_l}{\Rew_n}-\Ea{M\pi^0}{\Rew_n}}}+ 2\Tn\cdot\frac{\gamma^T(1-\gamma)}{1-\gamma^T} \geq \sum_{n=0}^{\infty}{\gamma^n \Ea{M\pi^0}{\V_n-\Q_{\pi^0n}}} \]

Applying Proposition~\ref{prp:regret_vq} to the right hand side and using equations (\ref{eqn:prp__short_long__eu_star}) and (\ref{eqn:prp__short_long__eu_zero}) we get the desired result.
\end{proof}

Given $(\Omega,P\in\Delta\Omega)$ a probability space, $A,B$ finite sets and $X: \Omega \rightarrow A$, $Y: \Omega \rightarrow B$ random variables, $\Ia{}{X;Y}$ denotes the mutual information between $X$ and $Y$. Given $C$ another finite set and $Z: \Omega \rightarrow C$ another random variable, $\CI{}{X;Y}{Z}: \Omega \rightarrow \Reals$ will denote the random variable obtained by first conditioning on $Z$ and then taking the mutual information between $X$ and $Y$, \emph{not} the expected value of this quantity, as sometimes used. $X_*P\in\Delta A$ denotes the pushforward of $P$ by $X$, i.e. the probability distribution of $X$. $P \mid X: \Omega \rightarrow \Delta\Omega$ denotes the conditional probability measure ($P$ conditioned on the value of $X$.) Given $\mu,\nu \in \Delta A$, $\KLD{\mu}{\nu}$ denotes the Kullb-Leibler divergence of $\mu$ from $\nu$.

\begin{samepage}
\begin{proposition}
\label{prp:delegation_information}

Consider $\A$ a finite set, $a^*\in\A$, $N \in \Nats^+$, $\epsilon \in \AP{0,{\Abs{\A}}^{-1}}$, and $\eta \in (0,1)$. Consider also $(\Omega,P)$ a probability space and random variables $K: \Omega \rightarrow [N]$ and $X: \Omega \rightarrow \A$. Suppose that for every $a \in \A$

\begin{equation}
\Pa{}{\CP{}{X=a}{K} > 0 \land \AP{a = a^* \lor \CP{}{X = a^*}{K} \leq \epsilon}} \leq 1 - \eta
\end{equation}

Then

\begin{equation}
\Ia{}{K;X} \geq \eta \ln\left(1 + \epsilon(1-\epsilon)^{\frac{1}{\epsilon}-1}\right)
\end{equation}

\end{proposition}
\end{samepage}

\begin{proof}

Define $q \in (0,1)$ by

$$q:=\frac{1}{\epsilon+(1-\epsilon)^{1-\frac{1}{\epsilon}}}$$

Let $a_q \in \A$ be s.t. $\Pa{}{X=a_q} > q\epsilon$ and either $a_q = a^*$ or $\Pa{}{X=a^*} \leq q\epsilon$. For every $k \in \Supp{K_*P}$, denote

$$\A_k := \ACM{a \in \A}{\CP{}{X=a}{K=k} > 0 \land \AP{a = a^* \lor \CP{}{X=a^*}{K=k} \leq \epsilon}}$$

If $a_q \not\in \A_k$ then either $\CP{}{X=a_q}{K=k}=0$ or both $\Pa{}{X=a^*}\leq q\epsilon$ and\\ $\CP{}{X=a^*}{K=k} > \epsilon$. In this case, conditioning by $K = k$ causes either the probability of $X = a_q$ to go down from at least $q\epsilon$ to $0$ or the probability of $X = a^*$ to go up from at most $q\epsilon$ to at least $\epsilon$. We get

$$\KLD{X_*(P \mid K = k)}{X_*P} \geq \min\AP{\KLD{0}{q\epsilon},\KLD{\epsilon}{q\epsilon}}$$

We have

\begin{align*}
\KLD{0}{q\epsilon} &= \ln{\frac{1}{1-q\epsilon}}\\
&=\ln{\frac{1}{1-\frac{\epsilon}{\epsilon+(1-\epsilon)^{1-\frac{1}{\epsilon}}}}}\\
&=\ln{\frac{\epsilon+(1-\epsilon)^{1-\frac{1}{\epsilon}}}{\epsilon+(1-\epsilon)^{1-\frac{1}{\epsilon}}-\epsilon}}\\ &=\ln{\AP{1+\epsilon(1-\epsilon)^{\frac{1}{\epsilon}-1}}} 
\end{align*}

\begin{align*}
\KLD{\epsilon}{q\epsilon} &= \epsilon \ln{\frac{\epsilon}{q\epsilon }}+(1-\epsilon)\ln{\frac{1-\epsilon}{1-q\epsilon}}\\
&= \epsilon \ln{\frac{1}{q}}+(1-\epsilon)\ln{(1-\epsilon)} +\ln{\frac{1}{1-q\epsilon}}- \epsilon \ln{\frac{1}{1-q\epsilon}}\\
&= \epsilon \ln{\frac{1-q\epsilon}{q}}+\ln{(1-\epsilon)^{1-\epsilon}} +\ln{\frac{1}{1-q\epsilon}}\\
&= \epsilon \ln{\AP{\frac{1}{q}-\epsilon}}+\ln{(1-\epsilon)^{1-\epsilon}} +\ln{\frac{1}{1-q\epsilon}}\\
&= \epsilon \ln{(1-\epsilon)^{1-\frac{1}{\epsilon}}}+\ln{(1-\epsilon)^{1-\epsilon}} +\ln{\frac{1}{1-q\epsilon}}\\
&= \ln{(1-\epsilon)^{\epsilon-1}}+\ln{(1-\epsilon)^{1-\epsilon}} +\ln{\frac{1}{1-q\epsilon}}\\
&=\ln{\frac{1}{1-q\epsilon}}\\
&=\ln{\AP{1+\epsilon(1-\epsilon)^{\frac{1}{\epsilon}-1}}}
\end{align*}

It follows that

\begin{align*}
\Ia{}{K;X} &= \Ea{}{\KLD{X_*(P \mid K)}{X_*P}} \\ 
&\geq \Pa{}{a_q \not \in \A_K} \ln{\AP{1+\epsilon(1-\epsilon)^{\frac{1}{\epsilon}-1}}} \\
&\geq \eta \ln{\AP{1+\epsilon(1-\epsilon)^{\frac{1}{\epsilon}-1}}}
\end{align*}
\end{proof}

In our notation, propositions about random variables are understood to hold almost surely.

\begin{samepage}
\begin{proposition}
\label{prp:delegation}

Consider non-empty finite sets $\A$ and $\B$, $N \in \Nats^+$, $\epsilon\in\AP{0,\Abs{\A}^{-1}}$, $\eta\in(0,1)$ and\\ $\AC{\Ad^k: \B \K \A}_{k\in[N]}$. Consider also  a probability space $(\Omega,P)$ and random variables $K: \Omega \rightarrow [N]$, $\AC{\bar{\Theta}_n: \Omega \rightarrow \B}_{n\in\Nats}$, $\AC{X_n,\Psi_n: \Omega \rightarrow \A \sqcup \{\bot\}}_{n \in \Nats}$, and $\AC{Z_n: \Omega \rightarrow \Delta[N]}_{n \in \Nats}$. Assume that for any $n\in\Nats$, $k\in[N]$ and $a\in\A$

\begin{enumerate}[i.]

\item\label{con:prp__delegation__xa} $\CP{}{X_{n+1} = a}{K,\bar{\Theta}_n,\Psi_n,Z_n}=\CP{}{X_{n+1}\ne\bot}{K,\bar{\Theta}_n,\Psi_n,Z_n} \Ad^K\APM{a}{\bar{\Theta}_n}$
\item\label{con:prp__delegation__xbot} $X_{n+1} = \bot \iff \exists a \in \A \forall k \in \Supp{Z_n}: \Ad^k\APM{a}{\bar{\Theta}_n} > 0 \land \AP{a=\Psi_n \lor \Ad^k\APM{\Psi_n}{\bar{\Theta}_n} \leq \epsilon}$
\item\label{con:prp__delegation__z} $Z_n(k)=\CP{}{K = k}{\bar{\Theta}_0,\bar{\Theta}_1\dots\bar{\Theta}_{n},\Psi_0,\Psi_1\dots\Psi_n,X_0,X_1 \dots X_{n}}$
\item\label{con:prp__delegation__eta} $Z_n(k) \geq \eta$

\end{enumerate}

Then,

\begin{equation}
\Ea{}{\Abs{\ACM{n\in\Nats^+}{X_n\ne\bot}}} \leq \frac{\ln{N}}{\eta\ln{\AP{1+\epsilon(1-\epsilon)^{\frac{1}{\epsilon}-1}}}}
\end{equation}

\end{proposition}
\end{samepage}

\begin{proof}

\begin{align*}
\ln{N} &\geq \Ea{}{\Ena{Z_0}} \\ 
&\geq \sum_{n=0}^\infty{\Ea{}{\Ena{Z_n}-\Ena{Z_{n+1}}}}\\
&= \sum_{n=0}^\infty{\Ea{}{\CE{}{\Ena{Z_n}-\Ena{Z_{n+1}}}{\Theta_n,\Psi_n,Z_n}}}\\
\end{align*}

Using assumption~\ref{con:prp__delegation__z}, we get

\begin{align}
\ln{N} &\geq \sum_{n=0}^\infty\Ea{}{\CI{}{K;\Theta_{n+1},\Psi_{n+1},X_{n+1}}{\Theta_n,\Psi_n,Z_n}}\nonumber\\ 
&\geq\sum_{n=0}^\infty\Ea{}{\CI{}{K;X_{n+1}}{\Theta_n,\Psi_n,Z_n}}\nonumber\\
\label{eqn:prp__delegation__proof__entropy}&\geq\sum_{n=0}^\infty\Ea{}{\CI{}{K;X_{n+1}}{\Theta_n,\Psi_n,Z_n};X_{n+1}\ne\bot}
\end{align}

Define the random variables $\AC{Q_{nak}: \Omega \rightarrow [0,1]}_{n \in \Nats, a \in \A, k\in[N]}$ by

\[Q_{nak}:=\CP{}{X_{n+1}=a}{K=k,\Theta_n,\Psi_n,Z_n}\]

Define the events $\{D_{nak} \subseteq \Omega\}_{n \in \Nats, a \in \A, k\in[N]}$ by

\[D_{nak}:=\AC{Q_{nak} > 0 \land \AP{a = \Psi_n \lor Q_{n\Psi_n k} \leq \epsilon}}\]

By assumption~\ref{con:prp__delegation__xbot}, the event $X_n=\bot$ is determined by $\Theta_n$, $\Psi_n$ and $Z_n$. Using assumption~\ref{con:prp__delegation__xa}, it follows that for any $n \in \Nats$, $a \in \A$ and $k \in [K]$

\[X_{n+1} \ne \bot \implies Q_{nak} = \Ad^k\APM{a}{\Theta_n}\]

\[X_{n+1} \ne \bot \implies \AP{D_{nak} \iff \Ad^k\APM{a}{\Theta_n} > 0 \land \AP{a = \Psi_n \lor \Ad^k\APM{\Psi_n}{\Theta_n} \leq \epsilon}}\]

Using assumption~\ref{con:prp__delegation__xbot}, we get

\[X_{n+1} \ne \bot \implies \exists k \in \Supp{Z_n}: \neg D_{nak}\]

Using assumption~\ref{con:prp__delegation__eta}

\[X_{n+1} \ne \bot \implies \Pa{k\sim Z_n}{D_{nak}} \leq 1 - \eta\]

Using assumption~\ref{con:prp__delegation__z}

\[X_{n+1} \ne \bot \implies \CP{}{D_{naK}}{\Theta_n,\Psi_n,Z_n} \leq 1 - \eta\]


Applying Proposition~\ref{prp:delegation_information} we conclude

\begin{equation}
\label{eqn:prp__delegation__proof__information}
X_{n+1} \ne \bot \implies \CI{}{K;X_{n+1}}{\Theta_n,\Psi_n,Z_n} \geq \eta \ln\left(1 + \epsilon(1-\epsilon)^{\frac{1}{\epsilon}-1}\right)
\end{equation}

Combining inequality~(\ref{eqn:prp__delegation__proof__entropy}) with inequality~(\ref{eqn:prp__delegation__proof__information}), we get

\[\ln{N} \geq \sum_{n=0}^\infty \Pa{}{X_{n+1}\ne\bot}\eta \ln\left(1 + \epsilon(1-\epsilon)^{\frac{1}{\epsilon}-1}\right)\]

Noticing that $\Ea{}{\Abs{\ACM{n\in\Nats^+}{X_n\ne\bot}}}=\sum_{n=0}^{\infty}\Pa{}{X_{n+1}\ne\bot}$, we get the desired result.
\end{proof}

Given a measurable space $X$ and $\mu,\nu\in\Delta X$, $\Dtva{\mu,\nu}$ will denote the total variation distance between $\mu$ and $\nu$.

\begin{samepage}
\begin{proposition}
\label{prp:thompson}

Consider a probability space $(\Omega, P)$, $N \in \Nats$, $\eta\in(0,1)$, $\zeta\in\Delta[N]$, a finite set $R \subseteq [0,1]$ and random variables $U: \Omega \rightarrow R$, $K: \Omega \rightarrow [N]$ and $J: \Omega \rightarrow [N]$. Assume that

\begin{enumerate}[i.]

\item\label{con:prp__thompson__distribution}$K_*P = J_*P = \zeta$
\item\label{con:prp__thompson__independent}$\Ia{}{K;J} = 0$
\item\label{con:prp__thompson__eta}$\forall k \in \Supp{\zeta}: \zeta(k) \geq \eta$

\end{enumerate}

Then,

\begin{equation}
\Ia{}{K;J,U} \geq 2 \eta \AP{\Ea{}{\CE{}{U}{K,J = K}}-\Ea{}{U}}^2
\end{equation}

\end{proposition}
\end{samepage}

\begin{proof}

Using the chain rule for mutual information

\[\Ia{}{K;J,U} = \Ia{}{K;J} + \Ea{}{\CI{}{K;U}{J}}\]

Using assumption~\ref{con:prp__thompson__independent}

\[\Ia{}{K;J,U} = \Ea{}{\CI{}{K;U}{J}}
= \Ea{}{\KLD{U_*\left(P \mid K,J\right)}{\ U_*\left(P \mid J\right)}}\]

Using Pinsker's inequality

$$\Ia{}{K;J,U} \geq 2\Ea{}{\Dtva{U_*\left(P \mid K,J\right),U_*\left(P \mid J\right)}^2} \geq 2\Ea{}{\AP{\Ea{}{U \mid K,J}-\Ea{}{U \mid J}}^2}$$

Denote $U_{kj} := \Ea{}{U \mid K = k, J = j}$. Using assumptions \ref{con:prp__thompson__distribution} and \ref{con:prp__thompson__independent}, we have

\begin{align*}
\Ia{}{K;J,U} &\geq 2\Ea{\substack{k\sim\zeta\\j\sim\zeta}}{\AP{U_{kj}-\Ea{k'\sim\zeta}{U_{k'j}}}^2}\\
&\geq 2\Ea{\substack{k\sim\zeta\\j\sim\zeta}}{\AP{U_{kj}-\Ea{k'\sim\zeta}{U_{k'j}}}^2;k=j}\\
&\geq 2\Ea{j \sim \zeta}{\zeta(j)\AP{U_{jj}-\Ea{k\sim\zeta}{U_{kj}}}^2}
&\end{align*}

Using assumption~\ref{con:prp__thompson__eta}

\begin{align*}
\Ia{}{K;J,U} \geq 2\eta\Ea{j \sim \zeta}{\AP{U_{jj}-\Ea{k\sim\zeta}{U_{kj}}}^2}\geq 2\eta \AP{\Ea{j \sim \zeta}{U_{jj}}-\Ea{\substack{k\sim\zeta\\j\sim\zeta}}{U_{kj}}}^2
\end{align*}

Using assumptions \ref{con:prp__thompson__distribution} and \ref{con:prp__thompson__independent} again, we get the desired result
\end{proof}

Given a proposition $\boldsymbol{\pi}$, the notation $[[\boldsymbol{\pi}]] \in \{0,1\}$ will mean 0 when the $\boldsymbol{\pi}$ is false and 1 when $\boldsymbol{\pi}$ is true.

\begin{proof}[Proof of Theorem~\ref{thm:regret_bound}]

The form of inequalities (\ref{eqn:thm__regret_bound__regret}) and (\ref{eqn:thm__regret_bound__delegations}) is s.t. we can assume w.l.o.g. that $\eta < \frac{1}{N}$ and $\epsilon < \frac{1}{\Abs{\A}}$.

\Comment{To avoid cumbersome notation, whenever $M^k$ should appear a subscript, we will replace it by $k$. }We are going to construct a probability space $(\Omega,P)$ and the random variables $K: \Omega \rightarrow [N]$ and for each $n\in\Nats$

\begin{align*}
Z^\dagger_n,\tilde{Z}^\dagger_n: \Omega &\rightarrow \Delta[N]\\ 
J^\dagger_n: \Omega &\rightarrow [N]\\
\Psi^\dagger_n: \Omega &\rightarrow \A_\X\\
A^\dagger_n: \Omega &\rightarrow \A_\X\\
X^\dagger_n: \Omega &\rightarrow \A_\X\\
\Theta^\dagger_n: \Omega &\rightarrow \St
\end{align*}

We also define $H^\dagger_n: \Omega \rightarrow \St^n$ by

$$H^\dagger_n:= \AP{\Theta^\dagger_0,\Theta^\dagger_1 \dots \Theta^\dagger_{n-1}}$$

By condition~\ref{con:def__sane__bold} of Definition~\ref{def:sane}, for each $k\in[N]$ we can choose some $\pi^{\star k}: \St \rightarrow \A$ s.t. for any $s \in \St$, $\pi^{\star k}(s) \in \A_{M^k}^\star(s)$ and $\Ad^k\APM{\pi^k(s)}{s} > \epsilon$.

We postulate that $K$ is uniformly distributed and for any $k \in [N]$, $l \in \Nats$, $m \in [T]$, $s\in\St$ and $a\in\A_\X$, denoting $n = lT+m$

\[A^\dagger_n = \begin{cases} \Psi^\dagger_n \text{ if } \forall k \in \Supp{Z^\dagger_n}: \Ad^k\APM{\Psi^\dagger_n}{\Theta^\dagger_n} > 0 \\ \text{some } a\in\A \text{ s.t. } \forall k \in \Supp{Z^\dagger_n}: \Ad^k\APM{a}{\Theta^\dagger_n} > 0 \text{ if such exists and } \Psi^\dagger_n=\bot\\\bot \text{ otherwise} \end{cases}\\\]
\begin{align*}
\tilde{Z}^\dagger_0(k)&=\frac{1}{N} \\ 
Z^\dagger_{n}(k) &= \frac{\tilde{Z}^\dagger_{n}(k)[[\tilde{Z}^\dagger_{n}(k) \geq \eta]] }{\sum_{j = 0}^{N-1}\tilde{Z}^\dagger_{n}(j)[[\tilde{Z}^\dagger_{n}(j) \geq \eta]]}\\
\CP{}{J^\dagger_{l} = k}{Z^\dagger_{lT}} &= Z^\dagger_{lT}\left(k\right)\\
\Psi^\dagger_{n} &= \begin{cases} \pi^{\star J^\dagger_l}\AP{\Theta^\dagger_n} \text{ if } Z^\dagger_n\AP{J^\dagger_l} > 0 \\ \bot \text{ otherwise} \end{cases}\\
\Theta^\dagger_0 &= s_0\\
X^\dagger_0 &= \bot\\
\CP{}{\Theta^\dagger_{n+1} = s,X^\dagger_{n+1}=a}{\Theta^\dagger_{n},A^\dagger_n} &= \T_{L^K}\APM{s,a}{\Theta^\dagger_n,A^\dagger_n}\\
\tilde{Z}^\dagger_{n+1}(k)\sum_{j = 0}^{N-1} Z^\dagger_n(j) \T_{L^j}\APM{\Theta^\dagger_{n+1},X^\dagger_{n+1}}{\Theta^\dagger_n,A^\dagger_n}&=Z^\dagger_{n}(k) \T_{L^k}\APM{\Theta^\dagger_{n+1},X^\dagger_{n+1}}{\Theta^\dagger_n,A^\dagger_n}
\end{align*}

Note that the last equation has the form of a Bayesian update which is allowed to be arbitrary when update is on "impossible" information.

This probability space can be constructed using standard arguments from the Kolmogorov extension theorem.

We now define $\pi^\dagger$ s.t. for any $n \in \Nats$, $a \in \A_\X$, $h \in \St_\X^n$ and $s\in\St_\X$

$$\Pr\left[H^\dagger_n=h,\AP{\Theta^\dagger_n,X^\dagger_n}=s\right] > 0 \implies \pi^\dagger\APM{a}{h,s}:=\CP{}{A^\dagger_n=a}{H^\dagger_n=h,\AP{\Theta^\dagger_n,X^\dagger_n}=s}$$

In order to prove $\pi^\dagger$ has the desired properties, we will define the stochastic processes $Z$, $\tilde{Z}$, $J$, $\Psi$, $A$, $X$ and $\Theta$, each process of the same type as its dagger counterpart (thus $\Omega$ is constructed to accommodate them.) These processes are required to satisfy the following:

\[A_n = \begin{cases} \Psi_n \text{ if } \forall k \in \Supp{Z_n}: \Ad^k\APM{\Psi_n}{\Theta_n} > 0\\ \text{some } a\in\A \text{ s.t. } \forall k \in \Supp{Z_n}: \Ad^k\APM{a}{\Theta_n} > 0 \text{ if such exists and } \Psi_n=\bot\\\bot \text{ otherwise} \end{cases}\\\]
\begin{align*}
\tilde{Z}_0(k)&=\frac{1}{N} \\ 
Z_{n}(k) &= \frac{\tilde{Z}_{n}(k)[[\tilde{Z}_{n}(k) \geq \eta]] }{\sum_{j = 0}^{N-1}\tilde{Z}_{n}(j)[[\tilde{Z}_{n}(j) \geq \eta]]}[[\tilde{Z}_{n}(K) \geq \eta]] \\&\ \ \ \ \ + [[K = k]]\cdot [[\tilde{Z}_{n}(K) < \eta]]\\
\CP{}{J_{l} = k}{Z_{lT}} &= Z_{lT}\left(k\right)\\
\Psi_{n} &= \begin{cases} \pi^{\star J_l}\AP{\Theta_n} \text{ if } Z_n\AP{J_l} > 0 \\ \bot \text{ otherwise} \end{cases}\\
\Theta_0 &= s_0\\
X_0 &= s_0\\
\CP{}{\Theta_{n+1} = s,X_{n+1}=a}{\Theta_{n},A_n} &= \T_{L^K}\APM{s,a}{\Theta_n,A_n}\\
\tilde{Z}_{n+1}(k)&=\frac{Z_{n}(k) \T_{L^k}\APM{\Theta_{n+1},X_{n+1}}{\Theta_n,A_n}}{\sum_{j = 0}^{N-1} Z_n(j) \T_{L^j}\APM{\Theta_{n+1},X_{n+1}}{\Theta_n,A_n}}
\end{align*}

As before, we also define $H_n:=\AP{\Theta_0,\Theta_1\dots\Theta_{n-1}}$.

We now construct $\AC{\IP\PoS}_{k\in[N]}$ s.t. for any $n \in \Nats$, $a \in \A_\X$, $h \in \St_\X^n$ and $s\in\St_\X$

\begin{align*}
&\Pr\left[H_n=h,\AP{\Theta_n,X_n}=s,K=k\right] > 0 \implies\\
&\pi^{!k}\APM{a}{h,s}:=\CP{}{A_n=a}{H_n=h,\AP{\Theta_n,X_n}=s,K=k}
\end{align*}

\Comment{...we define $\alpha_{\Ad\pi}: \St_\X^* \K \St^*$ by

$$\alpha_{\sigma\pi} (g \mid h) := [[h = \underline{g}]]C_h\prod_{n = 0}^{\Abs{h}-1} \sum_{a \in \A}\left([[g_n \in \bot a?]] \pi\left(\bot \mid g_{:n}\right)\sigma\left(a \mid h_{:n}\right)+[[g_n \in a\bot?]]\pi\left(a \mid g_{:n}\right)\right)$$

Here, $C_h \in \Reals$ is a constant defined s.t. the probabilities sum to 1. We define the $?$-policy $\left[\sigma\right]\underline{\pi}$ by

$$\left[\sigma\right]\underline{\pi}(a \mid h):=\Pr_{g \sim \alpha_{\sigma\pi}(h)}\left[\pi\left(g\right)=a \lor \left(\pi\left(g\right)=\bot \land \sigma(h)=a\right)\right]$$}

It is easy to see equation~(\ref{eqn:imaginary_safety}) holds, allowing us to apply Proposition~\ref{prp:short_long} and get

\[\Rg_{L^k}^{\IP}(\gamma) \leq \AP{1-\gamma^T}\sum_{n=0}^\infty \gamma^{nT}\AP{\EU^{\star k}_n-\EU^{!k}_n}+2\Tn_{M^k}(\gamma)\cdot\frac{1-\gamma}{1-\gamma^T}\]

Here, $\EU^{\star k}_n$ and $\EU^{!k}_n$ are defined according to equations (\ref{eqn:eustar}) and (\ref{eqn:eushriek}) respectively. We also define the $\pi^{\sharp k}_n\PoS$ by

$$\pi^{\sharp k}_n(a \mid h,s):=\begin{cases} \pi^{!k}(a \mid h) \text{ if } \Abs{h} < nT \\ \CP{}{A_n=a}{H_n=h,\AP{\Theta_n,X_n}=s,K=k,J_n=k} \text{ otherwise} \end{cases}$$

Defining $\EU^{\sharp k}_n$ according to equation~(\ref{eqn:eusharp}), we have

\[\Rg_{L^k}^{\IP}(\gamma) \leq \AP{1-\gamma^T}\sum_{n=0}^\infty \gamma^{nT}\AP{\EU^{\star k}_n-\EU^{\sharp k}_n+\EU^{\sharp k}_n-\EU^{!k}_n}+2\Tn_{M^k}(\gamma)\cdot\frac{1-\gamma}{1-\gamma^T}\]

Using equation~(\ref{eqn:sharp_vs_star_via_delegations}), we can apply Proposition~\ref{prp:delegation}. Indeed, conditions \ref{con:prp__delegation__xa}, \ref{con:prp__delegation__z} and \ref{con:prp__delegation__eta} are straightforward (for an appropriate definition of $\bar{\Theta}$.) To verify condition~\ref{con:prp__delegation__xbot}, consider two cases. In the case $Z_n\AP{J_l} > 0$ (where $l:=\Floor{\sfrac{n}{T}}$,) we have $\Psi_n\ne\bot$ and hence $A_n=\bot$ (equivalently $X_{n+1}\ne\bot)$ if and only if $\exists k \in \Supp{Z_n}: \Ad^k\APM{\Psi_n}{\Theta_n}=0$. This is equivalent to condition~\ref{con:prp__delegation__xbot} since, for any $a\ne\Psi_n$, taking $k=J_l$ makes the proposition false due to the fact that $\Ad^{J_l}\APM{\pi^{\star J_l}\AP{\Theta_n}}{\Theta_n} > \epsilon$ by construction of $\pi^{\star k}$. In the case $Z_n\AP{J_l} = 0$, we have $\Psi_n=\bot$ and hence $A_n=\bot$ (equivalently $X_{n+1}\ne\bot$) if and only if $\forall a\in\A\exists k \in \Supp{Z_n}:\Ad^k\APM{a}{\Theta_n}=0$. This is equivalent to condition~\ref{con:prp__delegation__xbot} since, in this case, $a\ne\Psi_n$ always. We get

\[\frac{1}{N}\sum_{k=0}^{N-1}\Rg_{L^k}^{\IP}(\gamma) \leq \frac{1-\gamma^T}{N}\sum_{k=0}^{N-1}\sum_{n=0}^\infty \gamma^{nT}\AP{\EU^{\sharp k}_n-\EU^{!k}_n}+O\AP{\bar{\Tn}\cdot\frac{1-\gamma}{1-\gamma^T}+\frac{\AP{1-\gamma^T}\ln{N}}{\eta^2\epsilon}}\]

Define the random variables $\AC{U_n : \Omega \rightarrow [0,1]}_{n\in\Nats}$ by 

$$U_n:=\frac{1-\gamma}{1-\gamma^T}\sum_{m=0}^{T-1} \gamma^{m} \Rew\AP{\Theta_{nT+m}}$$

\Comment{Denote 

$$\beta:=\frac{\AP{1-(1-\alpha)^T}\ln N}{\eta^2\epsilon}+\frac{\bar{\tau}\alpha}{1-(1-\alpha)^T}$$}

We get

\begin{align*}
\frac{1}{N}\sum_{k=0}^{N-1}\Rg_{L^k}^{\IP}(\gamma) \leq &\AP{1-\gamma^T}\sum_{n=0}^\infty \gamma^{nT} \Ea{}{\CE{}{U_n}{K, J_n = K, Z_{nT}}-\CE{}{U_n}{Z_{nT}}} \\ 
&+ O\AP{\bar{\Tn}\cdot\frac{1-\gamma}{1-\gamma^T}+\frac{\AP{1-\gamma^T}\ln{N}}{\eta^2\epsilon}}\\
\leq &\sqrt{\AP{1-\gamma^T}\sum_{n=0}^\infty \gamma^{nT} \Ea{}{\AP{\CE{}{U_n}{K, J_n = K, Z_{nT}}-\CE{}{U_n}{Z_{nT}}}^2}} \\ 
&+ O\AP{\bar{\Tn}\cdot\frac{1-\gamma}{1-\gamma^T}+\frac{\AP{1-\gamma^T}\ln{N}}{\eta^2\epsilon}}
\end{align*}

We apply Proposition~\ref{prp:thompson} to each term in the sum over $n$.

\begin{align*}
\frac{1}{N}\sum_{k=0}^{N-1}\Rg_{L^k}^{\IP}(\gamma) = &\sqrt{\AP{1-\gamma^T}\sum_{n=0}^\infty \gamma^{nT} \Ea{}{\frac{1}{2\eta}\CI{}{K;J_n,U_n}{Z_{nT}}}} \\
&+ O\AP{\bar{\Tn}\cdot\frac{1-\gamma}{1-\gamma^T}+\frac{\AP{1-\gamma^T}\ln{N}}{\eta^2\epsilon}}\\
\leq &\sqrt{\frac{1-\gamma^T}{2\eta}\sum_{n=0}^\infty \gamma^{nT} \Ea{}{\Ena{Z_{nT}}-\Ena{Z_{(n+1)T}}}} \\
&+ O\AP{\bar{\Tn}\cdot\frac{1-\gamma}{1-\gamma^T}+\frac{\AP{1-\gamma^T}\ln{N}}{\eta^2\epsilon}}\\
&=O\AP{\bar{\Tn}\cdot\frac{1-\gamma}{1-\gamma^T}+\sqrt{\frac{\AP{1-\gamma^T}\ln{N}}{\eta}}+\frac{\AP{1-\gamma^T}\ln{N}}{\eta^2\epsilon}}
\end{align*}

Thus, we derived equation~(\ref{eqn:regret_shriek}) and the rest of the proof can be completed as in appendix~\ref{sec:outline}.
\end{proof}

\begin{proof}[Proof of Corollary~\ref{crl:balanced_regret_bound}]

We set 

$$\eta:=(1-\gamma)^{1/4} N^{-1/2} \AP{\ln N}^{1/4} \AP{\frac{1}{\epsilon}+\Abs{\A}}^{1/4} \AP{\bar{\Tn}+1}^{1/4}$$  

$$T:=\Ceil{(1-\gamma)^{-1/4}N^{-1/2} \AP{\ln N}^{-1/4} \AP{\frac{1}{\epsilon}+\Abs{\A}}^{-1/4} \AP{\bar{\Tn}+1}^{3/4}}$$

By equation~(\ref{eqn:crl__balanced_regret_bound__gamma}), the expression we round to get $T$ is $\geq 1$, therefore this rounding can be absorbed within the constant factor. Equations (\ref{eqn:crl__balanced_regret_bound__regret}) and (\ref{eqn:crl__balanced_regret_bound__delegations}) follow straightforwardly\footnote{Note that there is an additional factor of $N$ since we now consider bounds for fixed $k\in[N]$ rather than for average over $k$.}.
\end{proof}

\subsubsection*{Acknowledgments}

This work was supported by the Machine Intelligence Research Institute (Berkeley, California, USA).

We wish to thank Alexander Appel for reviewing drafts of this work and providing helpful feedback.

\bibliography{DRL}
\bibliographystyle{iclr2019_conference}

\end{document}

%% file: math_commands.tex
\usepackage{amsmath,amsfonts,bm}

\def\eqref#1{equation~\ref{#1}}

\def\1{\bm{1}}

\DeclareMathAlphabet{\mathsfit}{\encodingdefault}{\sfdefault}{m}{sl}
\SetMathAlphabet{\mathsfit}{bold}{\encodingdefault}{\sfdefault}{bx}{n}

\newcommand{\KL}{D_{\mathrm{KL}}}

